\documentclass[letterpaper]{article}
\usepackage{aaai17}
\usepackage{times}
\usepackage{helvet}
\usepackage{courier}
\title{When to Reset Your Keys: Optimal Timing of Security Updates via Learning}

\title{When to Reset Your Keys: Optimal Timing of Security Updates via Learning}
\author{Zizhan Zheng\\
Department of Computer Science\\
Tulane University
\And Ness B. Shroff\\
Dept. of ECE and CSE\\
The Ohio State University
\And Prasant Mohapatra\\
Department of Computer Science\\
University of California, Davis
}

\usepackage[cmex10]{amsmath}
\usepackage[hyphens]{url}
\usepackage{amsthm, amssymb, graphicx, epsfig, epstopdf, color, setspace, subfigure, textcomp}
\usepackage{algorithm}
\usepackage[noend]{algpseudocode}

\makeatletter
\newcommand{\algmargin}{\the\ALG@thistlm}
\makeatother
\newlength{\whilewidth}
\algdef{SE}[parWHILE]{parWhile}{EndparWhile}[1]
  {\parbox[t]{\dimexpr\linewidth-\algmargin}{%
     \hangindent\whilewidth\strut\algorithmicwhile\ #1\ \algorithmicdo\strut}}{\algorithmicend\ \algorithmicwhile}%
\algnewcommand{\parState}[1]{\State%
  \parbox[t]{\dimexpr\linewidth-\algmargin}{\strut #1\strut}}

\newcommand{\ignore}[1]{}

\newtheorem{theorem}{Theorem}

\newtheorem{lemma}{Lemma}

\newtheorem{remark}{Remark}

\newif\iftp \tptrue

\begin{document}

\maketitle

\begin{abstract}
Cybersecurity is increasingly threatened by advanced and persistent attacks. As these attacks are often designed to disable a system (or a critical resource, e.g., a user account) repeatedly, it is crucial for the defender to keep updating its security measures to strike a balance between the risk of being compromised and the cost of security updates. Moreover, these decisions often need to be made with limited and delayed feedback due to the stealthy nature of advanced attacks. In addition to
targeted attacks, such an optimal timing policy under incomplete information has broad applications in cybersecurity. Examples include key rotation, password change, application of patches, and virtual machine refreshing. However, rigorous studies of optimal timing are rare. Further, existing solutions typically rely on a pre-defined attack model that is known to the defender, which is often not the case in practice. In this work, we make an initial effort towards achieving optimal timing of security updates in the face of unknown stealthy attacks. We consider a variant of the influential FlipIt game model with asymmetric feedback and unknown attack time distribution, which provides a general model to consecutive security updates.
The defender's problem is then modeled as a time associative bandit problem with dependent arms. We derive upper confidence bound based learning policies that achieve low regret compared with optimal periodic defense strategies that can only be derived when attack time distributions are known.
\end{abstract}

\section{Introduction}
Malicious attacks are constantly evolving to inflict increasing levels of damage on the nation's infrastructure systems, cooperate IT systems, and our digital lives. For example, the Advanced Persistent Threat (APT) has become a major concern to cybersecurity in the past few years. APT attacks exhibit two distinguishing behavior patterns~\cite{flipit} that make them extremely difficult to defend using traditional techniques. First, these attacks are often funded well and persistent. They attack a target system (or a critical resource) {\it periodically} with the goal to compromise it {\it completely} e.g., by stealing full cryptography keys. Second, the attacks can be highly adaptive. In particular, they often act {\it covertly}, e.g., by operating in a ``low-and-slow'' fashion~\cite{graph-apt}, to avoid immediate detection and obtain long-term advantages. 


From the defender's perspective, an effective way to thwart continuous and stealthy attacks is to update its security measures periodically to strike a balance between the risk of being compromised and the cost of updates. The primary challenge, however, is that such decisions must often be made with limited and delayed feedback because of the covert nature of the attacker. In addition to thwarting targeted attacks, such an {\it optimal timing} problem with incomplete information is crucial in various cybersecurity scenarios, e.g., key rotation~\cite{flipit}, password changes~\cite{Yue-2016}, application of patches~\cite{patch-usenix}, and virtual machine refreshing~\cite{flipit-patent}. 
For example, Facebook receives approximately 600,000 ``compromised logins'' from impostors every day~\cite{facebook-account}. An efficient approach to stop these attacks is to ask users to update their passwords when the risk of attack is high. 

Although time-related tactical security choices have been studied since the cold war era~\cite{games-of-timing}, rigorous study of timing decisions in the face of continuous and stealthy attacks is relatively new. In 2012, in response to an APT attack on it, the RSA lab proposed the FlitIt game, which was one of the first models to study timing decisions under stealthy takeovers. The FlipIt game model abstracts out details about concrete attack and defense operations by focusing on the stealthy and persistent nature of players. The basic model considers two players, each of whom can ``flip'' the state of a system periodically at any time with a cost. A player only learns the system state when she moves herself. The payoff of a player is defined as the fraction of time when the resource is under its control less the total cost incurred. 

The FlipIt game captures the stealthy behavior of players in an elegant way by allowing various types of feedback structures. In the basic model where neither player gets any feedback during the game and each move flips the state of the resource instantaneously, it is known that periodic strategies with random starting phases form a pair of best response strategies~\cite{flipit}. As a variant of the basic model, an asymmetric setting is studied in~\cite{asymmetric-model} where the defender gets no feedback during the game while the attacker obtains immediate feedback after each defense but incurs a random attack time to take over the resource. In this setting, it is shown in~\cite{asymmetric-model} that periodic defense and immediate attack (or no attack) form a pair of best response strategies. However, little is known beyond these two cases. In particular, designing adaptive defense strategies with partial feedback remains an open problem.

Although the FlipIt game provides a proper framework to understand the strategic behavior of stealthy takeover, it relies on detailed prior knowledge about the attacker. In particular, 
it requires parameters such as the amount of time needed to compromise a resource and the unit cost of each attack (or their distributions) to be fixed and known to the defender so that the equilibrium solution can be derived. These parameters limit the scope of the attack model, 
which, however, can be hard to verify before the game starts. To address this fundamental limitation, we propose to study
online learning algorithms that make minimum assumptions about the attacker and learn an optimal defense strategy from the limited feedback obtained during the game. Given the advances in big data analytics and their applications in cybersecurity, it is feasible for the defender to obtain partial feedback even under stealthy attacks. 
Such a learning approach makes it possible to derive adaptive and robust defense strategies against {\it unknown attacks} where the type of the attacker is derived from a fixed but unknown distribution, as well as the more challenging {\it dynamic attacks} where the type of the attacker can arbitrarily vary over time.


In this work, we make a first effort towards achieving optimal timing of security updates in the face of unknown stealthy attacks. We consider a variant of FlipIt game with asymmetric feedback similar to~\cite{asymmetric-model}, but with two key differences. 
First, we consider repeated unknown attacks with attacker's type sampled from an unknown distribution. Second, we assume that the defender obtains limited feedback about potential attacks at the end of each period. 
The defender's goal is to minimize the long-term cumulative loss. Our objective is to derive an adaptive defense policy that has a low regret compared with the optimal periodic defense policy when the attack time distribution is known. A key observation is that the set of defense periods that the defender can choose from are dependent in the sense that the loss from one defense period may reveal the potential loss from other periods, especially shorter ones. Moreover, two defense policies played for the same number of rounds may span different lengths of time, which has to be taken into account when comparing the policies. In this paper, we model the defender's problem as a time associate stochastic bandit problem with dependent arms, where each arm corresponds to one possible defense period. We derive optimal defense strategies for both the finite-armed bandit setting where the defense periods can only take a finite set of values, and the continuum-armed bandit setting where the defense periods can take any values from a non-empty interval.

Our main contributions can be summarized as follows.

\begin{itemize}
  \item We propose a stochastic time associative bandit model for optimal timing of security updates in the face of unknown attacks. Our model captures both the limited feedback about stealthy attacks and the dependence between different defense options. 
  \item We derive upper confidence bound (UCB) based policies for time associative bandits with dependent arms. Our policies achieve a regret of $O\left(\log(T(K+1))+K\right)$ for the finite arm case, where $T$ is the number of rounds played and $K$ is the number of arms, and a regret of $O(T^{2/3})$ for the continuous arm setting. 
\end{itemize}

Our learning model and algorithms are built upon the assumption that the defender can learn from frequent system compromises. This is reasonable for many online systems such as large online social networks and content providers and large public clouds, in which many customers are subject to similar attacks. In this setting, even if a single user is compromised occasionally, the system administrator can pool data collected from multiple users to obtain a reliable estimate quickly. For example, given the large number of attacks towards its users, Facebook can collect data from thousands of incidents of similar compromises in a short time. Our online learning algorithms can be used by Facebook to alert users to update their passwords when necessary.

\section{Related Work}
Time-related tactical security choices have been studied since the cold war era~\cite{games-of-timing}. However, the study of timing decisions in the face of continuous and stealthy attacks is relatively new. In particular, the FlipIt game~\cite{flipit} and its variants~\cite{asymmetric-model,FlipThem} are among the few models that study this problem in a rigorous way. However, all of these models assume that the parameters about the attacker are known to the defender at the beginning of the game. A gradient-based Bayesian learning algorithm was recently proposed in~\cite{Yue-2016} for a setting similar to ours, where the failure time was assumed to follow a Weibull distribution with one unknown parameter. In contrast, we consider a general attack time distribution.

Multi-armed bandit problems have been extensively studied for both the stochastic setting and the adversarial setting~\cite{Bubeck-bandit-survey}. Many variants of bandit models have been considered including bandits with side observations~\cite{side-stochastic-12,Swapna-Sigmetrics2014}. In the context of cybersecurity, bandit models have been applied to anomaly detection~\cite{probing-13} and stackelberg security games~\cite{online-learning-SS}. However, the only previous work that studies the time associative bandit model is ~\cite{continuous-time-bandit}, where the arms are assumed to be mutually independent. In contrast, we propose to model the optimal timing problem in cybersecurity as a time associative bandit problem with dependent arms and study algorithms that can exploit side-observations to improve performance.

\section{Model}\label{sec:model}

We consider the following variant of the FlipIt game~\cite{flipit} with two players, a defender and an attacker, and a security sensitive resource to protect. The attacker is persistent in compromising the resource. In response, the defender updates its security measures, e.g., keys, passwords, etc., from time to time to thwart the attacker. Assume a continuous time horizon. At any time instance, either the defender or the attacker can make a move to take over the resource at some cost. At time $\tau$, the resource is under the control of the player that makes the last move before $\tau$. Let $\tau_t, t = 1, 2, ...$ denote the time instance of the $t$-th defense action, and $x_t = \tau_{t+1} - \tau_t$ the $t$-th defense period. We assume $\tau_1 = 0$ without loss of generality. Each defense action (move) incurs a fixed cost $C_D$, which is known to the defender. Let $X$ denote the set of all possible defense periods. We assume that $X \subseteq [x_{\min}, x_{\max}]$ with $x_{\min}>0$. See Figure~\ref{Fig.model} for an example.


We further make the following assumptions about the game: (1) The attack in round $t$ takes a random time $a_t$ to succeed, which is $i.i.d.$ sampled from a distribution $F_a$ that is initially unknown to the defender. In contrast, the defender recaptures the resource immediately once it makes a move, which is a reasonable assumption as it is usually much more time consuming to compromise a resource than updating its security measures. We may also interpret $a_t$ as the awareness time that the attacker takes to discover a new vulnerability in the system. We assume that $a_t$ is out of the control of the attacker but its distribution is known to the attacker. The attacker does not know the value of $a_t$ until it successfully compromises the system in round $t$. (2) Whenever the defender makes a move, this fact is learned by the attacker immediately. 
On the other hand, the defender has delayed and incomplete feedback in the following sense. First, the defender only gets feedback at the end of each round. Second, at time $\tau_{t+1}$, the defender learns the value of $a_t$ if $a_t < x_t$, that is only when an attack is observed.
Thus, the game has asymmetric feedback, a common scenario in cybersecurity. (3) The attacker is myopic and does not have a move cost. Therefore, it always attacks immediately right after a security update. Our model and solutions can be extended to the case when the attacker is myopic but with a hidden move cost \iftp(see the Appendix).\else (see~\cite{technical-report}).\fi


\begin{figure}[t]
\centering
\includegraphics[width=3.2in]{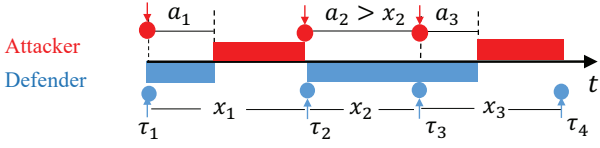}
\caption{\small An example of the proposed game model. Blue circles and red circles
represent the defender's and the attacker's actions, respectively. A blue
segment denotes an interval where the resource is under protection, and a red
segment denotes an interval where the resource is compromised.}
\label{Fig.model}
\end{figure}

Under the above assumptions, in each round $t$, the resource is fully protected if $a_t \geq x_t$ and is compromised for a duration of $x_t-a_t$ otherwise. The loss to the defender in round $t$ is then defined as:
\begin{align}
l(x_t,a_t) = f[(x_t-a_t)^+]+c_d \label{eq:loss1}
\end{align}
where $f(\cdot)$ models the loss from attack and $c_d$ models the cost of each defense action. We assume that $f(\cdot) \in [0,1]$, $f(0) = 0$, and $f(\cdot)$ is increasing. 
For instance, we can consider (1) a binary loss function where $l(x_t,a_t) = 1+ c_d$ if $x_t > a_t$ and $l(x_t,a_t) = c_d$ otherwise; or (2)
a linear loss function $l(x_t,a_t) = \frac{(x_t-a_t)^+}{x_{\max}} +c_d$ (where the $x_{\max}$ factor is introduced to normalize the loss value). 

The defender's objective is to minimize the long-term average loss defined as follows:
\begin{align}
\lambda^u = \limsup_{T \rightarrow \infty} \frac{\mathbb{E}(\sum^{T}_{t=1}l(x^u_t,a_t))}{\sum^T_{t=1}x^u_t}
\end{align}
where $u$ denotes any defense policy and $x^u_t$ is the $t$-th defense period chosen by policy $u$. Let $l(x) = \mathbb{E}_{a_1}(l(x,a_1))$ denote the expected loss of defense period $x$, and let $\lambda(x) = \frac{l(x)}{x}$ denote the time average loss of a {\it periodic} policy with period $x$. We make two observations: (1) defending all the time is not necessarily a good option as it may incur a very high defense cost; (2) it can be shown that the periodic defense policy with period $x^* = \min_{x \in X} \lambda(x)$ minimizes the long-term time average loss~\cite{Puterman-MDP,continuous-time-bandit}.
However, this optimal policy cannot be found when the distribution of $a_t$ is unknown. Let $\lambda^* = \lambda(x^*)$ denote the optimal loss. To find an optimal defense policy when the distribution of $a_t$ is unknown, we adopt the time associative bandit model~\cite{continuous-time-bandit} by considering each defense period as an arm. For a defense policy $\{x_t\}$, the (pseudo) regret for the first $T$ rounds with respect to the optimal periodic policy can be defined as:
\begin{align}
\overline{R}_T &= \max_{x \in X} \mathbb{E}\left[\sum^{T}_{t=1}l(x_t,a_t) - \lambda(x) x_t\right]\\
&= \sum^{T}_{t=1}l(x_t)- \lambda^* \sum^T_{t=1}x_t
\end{align}

Our objective is to find a defense policy with low regret. Note that any learning algorithm with $\lim_{T \rightarrow \infty} \frac{\overline{R}_T}{T} = 0$ minimizes the long-term loss as $T \rightarrow \infty$. We also note that even if $l(x_t,a_t)$ as a function of $x_t$ (for a fixed $a_t$) has a simple structure, the mean loss function $l(x_t) \triangleq E_{a_t}(l(x_t,a_t))$ may have a complicated form depending on the distribution of $a_t$. Therefore, previous works on linear and convex bandits cannot be directly applied to our problem. On the other hand, we observe that the defender may obtain side-observations during the game, which can be utilized to design more efficient learning algorithms. 


\vspace{1ex}
\noindent{\bf Side observations:} As we discussed before, the defender learns the value of $a_t$ if $a_t < x_t$ (hence its loss as well) at the end of each round. From this feedback, the defender may get side observations in the following sense. 
Consider any round $t$. If $a_t < x_t$, 
then the defender learns the value of $a_t$; therefore, it learns $l(x_i,a_t)$ for any $x_i \in X$ if it has played $x_i$ instead of $x_t$. On the other hand, if $a_t \geq x_t$, the defender only learns the value of $l(x_i,a_t) = c_d$ for any $x_i \leq x_t$, but not the value of $l(x_j,a_t)$ for $x_j > x_t$.
This implies that playing an arm that corresponds to a longer defense period provides more side observations about other arms. 
Our learning algorithm incorporates these side-observations to minimize the expected regret. Indeed, our algorithm and its regret bound apply to any loss function $l(x_t,a_t)$ where playing one period provides side-observations to all shorter periods. 

\vspace{1ex}
\noindent{\bf Multiple resources:} Our model can be readily extended to consider multiple resources (nodes) subject to $i.i.d.$ attacks, which can be used to model multiple users subject to independent attacks in an online system such as Facebook. In this case, samples from multiple nodes can be pooled together when choosing the next defense period for a node. Consider a system with $N$ nodes that are subject to $i.i.d.$ attacks with unknown attack times sampled from $F_a$. Let $\tau_{st}$ denote the time instance of the $t$-th security update on node $s$ and $x_{st} = \tau_{s(t+1)}-\tau_{st}$ the $t$-th defense period for node $s$. Note that $x_{st}$ can be different for different $s$. Let $a_{st}$ denote the attack time in the $t$-th attack towards node $s$. Let $l(x_{st},a_{st})$ denote the loss to the defender in round $t$ over node $s$. When the nodes are subject to $i.i.d.$ attacks, there is an optimal defense period $x^*$ for all nodes with minimum time average loss $\lambda^*$, similar to the single node setting. The $i.i.d.$ assumption may hold in practice because (1) some parameters such as the attack time may be out of the control of the attacker and can be approximated as $i.i.d.$ random variables during the time horizon of the game; (2) an adversarial attacker may choose to avoid correlated attacks to make its behavior more unpredictable. Assume that the game is played for $T_s$ rounds over node $s$, and let $T = \sum_s T_s$. Then the regret over the $T$ rounds of play across all the nodes can be defined as $\overline{R}_T = \sum^N_{s=1} \sum^{T_s}_{t=1}l(x_{st})- \lambda^* \sum^N_{s=1}\sum^{T_s}_{t=1}x_{st}$. Note that, when choosing $x_{st}$, feedback from all the nodes received before $\tau_{st}$ can be used. Our online learning algorithms can be directly applied to this setting. 

\section{Optimal Timing Algorithms}
In this section, we present our learning algorithms for the optimal timing problem for both discrete and continuous defense periods. 

\subsection{Discrete Defense Periods}
We first consider the finite-armed setting where the set of defense periods is finite, denoted by $X = \{x_1,...,x_K\}$. 
Let $i_t$ denote the index of the arm played in round $t$, i.e., $x_{i(t)}$ is the defense period chosen for round $t$. Let $n_i(t) = \sum^t_{s=1}\mathbb{I}(i_s = i)$ denote the number of plays of arm $i$ during the first $t$ rounds. Let $\overline{l}_{i,t} = \frac{1}{n_i(t)} \sum^t_{s=1}\mathbb{I}(i_s = i)l(x_{i(s)},a_s)$ denote the average loss from arm $i$ during the first $t$ rounds, and $\overline{\lambda}_{i,t} = \frac{\overline{l}_{i,t}}{x_i}$ the time average loss of arm $i$. To simplify the notation, we omit the subscript $t$ in $\overline{l}_{i,t}$ and $\overline{\lambda}_{i,t}$ when it is clear from the context, and let $l_i \triangleq l(x_i)$ and $\lambda_i \triangleq \lambda(x_i)$. Let $\Delta_i = l_i - x_i \lambda^*$ denote the {\it relative} loss of playing arm $i$. Note that $\Delta_{i^*} = 0$ for an optimal arm $i^*$ and $\Delta_i \in [0,1]$ by our assumption about $l$. Then $\overline{R}_T=\sum_{i \neq i^*} \Delta_i \mathbb{E}(n_i(T))$. Let $\Delta_{\min} \triangleq \min_{i:\Delta_i>0} \Delta_i$ and $\Delta_{\max} \triangleq \max_i \Delta_i$.


\begin{algorithm}[t]
\caption{Improved UCB algorithm for time-associative bandits with side observations}\label{alg:stochastic-finite} 
\begin{algorithmic}
\State {\bf Input:} A set of periods $X$, the number of rounds $T$.
\State {\bf Initialization:} Set $\tilde{\Delta}_0 = 1, X_0 = X$.
\For{$m = 0,1,2,...,$}
\State $x_{(1)} = \min \{x_i \in X_m\}$; $x_{(2)} = \max \{x_i \in X_m\}$.
\vspace{0.5ex}
\State {\bf Arm selection:}
\State If $|X_m| = 1$, play the single period in $X_m$ until $T$.
\parState{Else play the longest period in $X_m$ until round $\min(n_m,T)$, where $n_m = \left\lceil \frac{2\gamma_m\log(T(K+1)\tilde{\Delta}^2_m)}{\tilde{\Delta}^2_m} \right\rceil$ and $\gamma_m = \left(1+\frac{x_{(2)}}{x_{(1)}}\right)^2$; update $\overline{l}_i, \overline{\lambda}_{i}$ for all $x_i \in X_m$.} 
\vspace{0.5ex}
\State {\bf Arm elimination:}
\parState{$\overline{\lambda}_m = \min_{x_i \in X_m} \left(\overline{\lambda}_{i}+c_m/x_i\right)$ where $c_m = \sqrt{\frac{\log(T(K+1)\tilde{\Delta}^2_m)}{2n_m}}$.}
\parState{To get $X_{m+1}$, delate all the periods $x_i \in X_m$ such that $\overline{l}_{i}-x_i\overline{\lambda}_{m} \geq \min_{x_j \in X_m}\overline{l}_{j}-x_j\overline{\lambda}_{m}+2\left(1+\frac{x_j}{x_{(1)}}\right)c_m$.}
\vspace{0.5ex}
\State {\bf Reset:} $\tilde{\Delta}_{m+1} = \frac{\tilde{\Delta}_m}{2}$.
\EndFor
\end{algorithmic}
\end{algorithm}

To derive an optimal defense policy, 
we consider the following variant of the improved upper confidence bound based policy proposed in~\cite{improved-UCB} for stochastic bandits. We modify the improved UCB policy to address the time associative regret while taking the dependence between arms into account. 

The algorithm proceeds in multiple stages, where each stage involves multiple rounds (see Algorithm~\ref{alg:stochastic-finite}). In each stage $m$, as in the improved UCB policy, our policy estimates $\Delta_i$ by a value $\tilde{\Delta}_m$, and maintains a set of active arms $X_m$. $\tilde{\Delta}_0$ is initialized to 1 and is halved in each stage. $X_0$ initially contains all the arms. At the end of each stage $m$, a subset of arms are deleted from $X_m$ according to their observed losses in previous rounds. Compared with the improved UCB policy for stochastic bandits, our policy has several key differences. First, in the arm selection phase, each active arm is played $n_m-n_{m-1}$ times in stage $m$ in the improved UCB policy, where $n_m$ is a function of $\tilde{\Delta}_m$ and is chosen so that any suboptimal arm $i$ is eliminated as soon as $\tilde{\Delta}_m < \frac{\Delta_i}{2}$ with high probability. In contrast, only the longest period in $X_m$ is played $n_m-n_{m-1}$ times in our policy, which provides side observations to all the shorter periods as we discussed above. For any arm $x_i \in X_m$, $\overline{l}_i$ is defined as if $i$ is played in all the previous $n_m$ rounds. In addition, the definition of $n_m$ in our policy is different from the improved UCB policy. In particular, $n_m$ depends on the ratio of the maximum active period to the minimum active period, which is needed to bound the time associative regret. Second, in the arm elimination phase, we compare the relative losses of arms instead of average losses as in the improved UCB, since average loss alone does not take the length of a defense period into account. In particular, we estimate the relative loss of arm $i$ by $\overline{l}_{i,n_m}-x_i\overline{\lambda}_{m}$, where $\overline{\lambda}_{m}$ is an estimate of $\lambda^*$ defined by
\begin{align}
\overline{\lambda}_m = \min_{x_i \in X_m} \left(\overline{\lambda}_{i,n_m}+c_m/x_i\right)
\end{align}
where $c_m = \sqrt{\frac{\log(T(K+1)\tilde{\Delta}^2_m)}{2n_m}}$. 
The value of $c_m$ is chosen so that for all $i$, $\overline{\lambda}_{i,n_m}$ is in the $c_m/x_i$-vicinity of $\lambda_i$ with high probability.


To establish the regret bound of the algorithm, we need the following lemmas.

\begin{lemma}({\it Chernoff-Hoeffding Bound~\cite{Hoeffding-inequalities}})
Let $X_1, X_2,...,X_n$ be a sequence of independent random variables with support $[a,b]$ and $E(X_t) = \mu$ for all $X_t$. Let $\overline{X}_n = \frac{1}{n}\sum^n_{t=1}X_t$. Then for any $\epsilon>0$, we have
\begin{align*}
&\mathbb{P}\{\overline{X}_n \geq \mu + \epsilon\} \leq e^{-\frac{2n\epsilon^2}{(b-a)^2}},\\
&\mathbb{P}\{\overline{X}_n \leq \mu - \epsilon\} \leq e^{-\frac{2n\epsilon^2}{(b-a)^2}}.\\
\end{align*}
\end{lemma}

\begin{lemma}\label{lem:1}
Consider any stage $m$ where there is an optimal arm $i^* \in X_m$. If $l_i \leq \overline{l}_{i} + c_m$ for all $x_i \in X_m$ and $l^* \geq \overline{l}_{i^*} - c_m$, then we must have $\lambda^* \leq \overline{\lambda}_m \leq \lambda^* +2c_m/x_{i^*}$.
\end{lemma}
\begin{proof}
To see this, let $x_j \in X_m$ be the arm that minimizes $\overline{\lambda}_{j}$. Then we have $\lambda^* \leq \lambda_j \leq \overline{\lambda}_{j} +c_m/x_j = \overline{\lambda}_m$, and $\overline{\lambda}_m \leq \overline{\lambda}_{i^*} + c_m/x_{i^*} \leq \lambda^* + 2c_m/x_{i^*}$.
\end{proof}

We now show the following bound on the expected regret of Algorithm~\ref{alg:stochastic-finite}.
\begin{theorem}\label{thm:finite}
The expected regret of Algorithm~\ref{alg:stochastic-finite} is at most $\frac{48\gamma\log\left(T(K+1)\frac{\Delta^2_{\max}}{4}\right)}{\Delta_{\min}}+\sum_{i: \Delta_i>0}\left(\Delta_i+\frac{48}{\Delta_i}\right)$, where $\gamma = (1+\frac{x_{\max}}{x_{\min}})^2$. 
\end{theorem}
\begin{proof}
Without loss of generality, we assume that the optimal arm is unique and has index $K$, and sort the set of arms such that $\Delta_1 \geq \Delta_2 \geq \cdots \geq \Delta_K = 0$. For any suboptimal arm $i$, let $m_i = \min\{m: \tilde{\Delta}_m < \frac{1}{2}\Delta_i\}$ denote the first stage in which $\tilde{\Delta}_m < \frac{1}{2}\Delta_i$. We have $2^{m_i} = \frac{1}{\tilde{\Delta}_{m_i}} \leq \frac{4}{\Delta_i} < \frac{1}{\tilde{\Delta}_{m_i+1}}=2^{m_i+1}$. Note that $m_1 \leq m_2 \leq \cdots \leq m_{K-1}$. 

We consider the following events similar to~\cite{MAB-covariates}. Let $A_i$ denote the event that the optimal arm has {\it not} been eliminated before stage $m_i$, and $B_i$ the event that every arm $j \in \{1,2,...,i\}$ has been eliminated in stage $m_j$ or before. We have $A_1 \supseteq A_2 \cdots \supseteq A_{K-1}$ and $B_1 \supseteq B_2 \cdots \supseteq B_{K-1}$. Let $C_i = A_i \cap B_i$ for $i \in \{1,2,...,K-1\}$. Under the event $C_i$, let $U_i$ denote the contribution to the regret from arms $\{1,2,...,i\}$ and $V_i$ the contribution to the regret from arms $\{i+1,...,K-1\}$. We observe that $V_i \leq T\Delta_{i+1}$. 
Let $C_0$ denote the sample space. We then have
\begin{align}
\overline{R}_T &= \sum^{K-1}_{i=1}\mathbb{P}(C_{i-1}\backslash C_i) (U_i+V_i) \nonumber\\
&\leq \sum^{K-1}_{i=1}U_i\mathbb{P}(C_{i-1}\backslash C_i) +\sum^{K-1}_{i=1}T\Delta_i\mathbb{P}(C_{i-1}\backslash C_i) \nonumber \\
&\leq \sum^{K-1}_{i=1}U_i\mathbb{P}(C_{i-1}\backslash C_i) +\sum^{K-1}_{i=1}T\Delta_i\mathbb{P}(B^c_i \cap B_{i-1} \cap A_i) \nonumber \\
&\hspace{20ex}+\sum^{K-1}_{i=1}T\Delta_i\mathbb{P}(A^c_i \cap C_{i-1}) \label{regret-terms}
\end{align}

\noindent We bound each of the three terms in~\eqref{regret-terms} as follows:




\vspace{0.5ex}
\noindent{\bf First term in~\eqref{regret-terms}}: Under the event $C_i$, each suboptimal arm $j \in \{1,2,...,i\}$ is eliminated on or before round $n_{m_j} = \left\lceil \frac{2\gamma_{m_j}\log(T(K+1)\tilde{\Delta}^2_{m_j})}{\tilde{\Delta}^2_{m_j}} \right\rceil$. 
Among these arms, let $j_1, j_2, ..., j_k$ denote the sequence of suboptimal arms played where $x_{j_1} > ... > x_{j_k}$, and arm $j_i$ is eliminated in stage $m'_{j_i} \leq m_{j_i}$. 
Let $B \triangleq 2\gamma\log\left(T(K+1)\frac{\Delta^2_{\max}}{4}\right)$. We then have 
\vspace{-1ex}
\small{
\begin{align*}
U_i& \leq \Delta_{j_1} n_{m'_{j_1}} + \sum^k_{i=2} \Delta_{j_i} (n_{m'_{j_i}}-n_{m'_{j_{i-1}}})\\
\leq& \Delta_{j_1} n_{m'_{j_1}} + \sum^k_{i=2}\Delta_{j_i} \Bigg(1+ \frac{2\gamma\log(T(K+1)\tilde{\Delta}^2_{m_{j_i}})}{\tilde{\Delta}^2_{m_{j_i}}}\\
&\hspace{25ex}-\frac{2\gamma\log(T(K+1)\tilde{\Delta}^2_{m_{j_{i-1}}})}{\tilde{\Delta}^2_{m_{j_{i-1}}}} \Bigg)  \\
\leq& \Delta_{j_1} n_{m'_{j_1}}+\sum^k_{i=2} \Delta_{j_i} + 4B\sum^k_{i=2}\tilde{\Delta}_{m_{j_i}} \left(\frac{1}{\tilde{\Delta}^2_{m_{j_i}}}- \frac{1}{\tilde{\Delta}^2_{m_{j_{i-1}}}}\right)  \\
\leq& \Delta_{j_1} n_{m'_{j_1}}+\sum^k_{i=2} \Delta_{j_i} + 4B\sum^k_{i=2} \frac{1.5\tilde{\Delta}_{m_{j_{i-1}}}(\tilde{\Delta}_{m_{j_{i-1}}}-\tilde{\Delta}_{m_{j_i}})}{\tilde{\Delta}_{m_{j_i}}\tilde{\Delta}^2_{m_{j_{i-1}}}} \\
=& \Delta_{j_1} n_{m'_{j_1}}+\sum^k_{i=2} \Delta_{j_i} + 6B\sum^k_{i=2} \left(\frac{1}{\tilde{\Delta}_{m_{j_i}}}-\frac{1}{\tilde{\Delta}_{m_{j_{i-1}}}}\right) \\
\leq& \sum^k_{i=1} \Delta_{j_i} + 6B \frac{1}{\tilde{\Delta}_{m_{j_k}}} \leq \sum^k_{i=1} \Delta_{j_i} + 24B \frac{1}{\Delta_{j_k}}
\end{align*}
}

\noindent Therefore, $\sum^{K-1}_{i=1}U_i\mathbb{P}(C_{i-1}\backslash C_i) \leq \sum^{K-1}_{i=1} \Delta_i + 24B \frac{1}{\Delta_{\min}}$.

\vspace{0.5ex}
\noindent{\bf Second term in~\eqref{regret-terms}}: Under the event $B^c_i \cap B_{i-1} \cap A_i$, the optimal arm is not eliminated by $m_i$, neither does arm $i$. We first note that if $\overline{l}_{i} \geq l_i - c_{m_i}$ and $\overline{l}_K \leq l_K + c_{m_i}$ hold, then arm $i$ will be eliminated in round $m_i$. 
Indeed, from the definitions of $c_m$ and $n_m$, we have $c_{m_i} \leq \frac{\tilde{\Delta}_{m_i}}{2\sqrt{\gamma_{m_i}}} = \frac{\tilde{\Delta}_{m_{i+1}}}{{\sqrt{\gamma_{m_i}}}} < \frac{\Delta_i}{4\sqrt{\gamma_{m_i}}}$. Then from Lemma~\ref{lem:1}, we have
\begin{align*}
\overline{l}_{i} - x_i\overline{\lambda}_{m_i} \geq&\overline{l}_i - x_i(\lambda^*+2c_{m_i}/x_K)\\
\geq& l_i - x_i(\lambda^*+2c_{m_i}/x_K)- c_{m_i}\\
=& l_K-x_K\lambda^*+ \Delta_i - 2\frac{x_i}{x_K}c_{m_i} - c_{m_i}\\
>& l_K-x_K\lambda^*+4\sqrt{\gamma_{m_i}}c_{m_i} - 2\frac{x_i}{x_K}c_{m_i} - c_{m_i}\\
\geq& \overline{l}_K-x_K\lambda^* + 2\left(1+\frac{x_K}{x_{(1)}}\right)c_{m_i}\\
\geq& \overline{l}_K-x_K\overline{\lambda}_{m_i} + 2\left(1+\frac{x_K}{x_{(1)}}\right)c_{m_i}
\end{align*}
\noindent where $x_{(1)}$ is the minimum active period in stage $m_i$. It follows that arm $i$ is eliminated in stage $m_i$ as claimed.

It follows that $\mathbb{P}(B^c_i \cap B_{i-1} \cap A_i) \leq \mathbb{P}(\overline{l}_{i} < l_i - c_{m_i}) + \mathbb{P}(\overline{l}_{i^*} > l_{i^*} + c_{m_i}) \leq \frac{1}{T(K+1)\tilde{\Delta}^2_{m_i}} + \frac{1}{T(K+1)\tilde{\Delta}^2_{m_i}} \leq \frac{1}{T\tilde{\Delta}^2_{m_i}}$ by the Chernoff-Hoeffding bound. 
Therefore, the second term in~\eqref{regret-terms} can be bounded by $\sum_i T\Delta_i \frac{1}{T\tilde{\Delta}^2_{m_i}} \leq \sum_i \frac{16}{\Delta_i}$.

\vspace{0.5ex}
\noindent{\bf Third term in~\eqref{regret-terms}}: Under the event $A^c_i \cap C_{i-1}$, every arm $j \in \{1,2,...i-1\}$ has been eliminated by stage $m_j$ and the optimal arm is eliminated by some arm $k \geq i$ in some stage $m_*$ where $m_{i-1}<m_* \leq m_i$. We first claim that if $\overline{l}_{k} \geq l_k - c_{m_*}$ and $\overline{l}_{K} \leq l_{K} + c_{m_*}$ hold, then the optimal arm is {\it not} eliminated by arm $k$ in stage $m_*$. To see this, assume that the optimal arm is eliminated, which happens only when $\overline{l}_{K} -x_{K}\overline{\lambda}_{m_*} \geq \overline{l}_k - x_k\overline{\lambda}_{m_*} + 2\left(1+\frac{x_k}{x_{(1)}}\right)c_{m_*}$. From Lemma~\ref{lem:1}, we have:
\begin{align*}
l_k - x_k\lambda^* \leq&\overline{l}_k +c_{m_*} - x_k(\overline{\lambda}_{m_*}-2c_{m_*}/x_{K})\\
\leq& \overline{l}_k -x_k\overline{\lambda}_{m_*}+ c_{m_*}+2\frac{x_k}{x_{K}}c_{m_*}\\
\leq& \overline{l}_k -x_k\overline{\lambda}_{m_*} +2\left(1+\frac{x_k}{x_{(1)}}\right)c_{m_*}-c_{m_*} \\
\leq& \overline{l}_{K} -x_{K}\overline{\lambda}_{m_*} -c_{m_*} \\
\leq& l_{K} -x_{K}\lambda^*
\end{align*}
\noindent
which contradicts the fact that $k$ is suboptimal. It follows that the probability that the optimal arm is eliminated by a fixed arm $k \geq i$ in a fixed stage $m_* \leq m_i$ is bounded by $\mathbb{P}(\overline{l}_{k} < l_k - c_{m_*})+ \mathbb{P}(\overline{l}_{K} > l_{K} + c_{m_*}) \leq \frac{1}{T\tilde{\Delta}^2_{m_{*}}}$ by the Chernoff-Hoeffding bound. Therefore, the third term in~\eqref{regret-terms} is bounded by
\begin{align*}
&\sum^{K-1}_{i=1} \sum^{m_i}_{m_*=m_{i-1}+1} \sum^{K-1}_{k = i}  \frac{1}{T\tilde{\Delta}^2_{m_*}} T\Delta_i \\
&= \sum^{\max_i m_i}_{m_*=0} \sum_{k: m_k \geq m_{*}} \frac{1}{T\tilde{\Delta}^2_{m_*}} T\max_{h: m_h \geq m_*}\Delta_h \\
&\leq \sum^{\max_i m_i}_{m_*=0} \sum_{k: m_k \geq m_*} \frac{1}{\tilde{\Delta}^2_{m_*}} 4 \tilde{\Delta}_{m_*} \\
&= \sum^{K-1}_{i=1} \sum^{m_i}_{m_*=0} \frac{4}{\tilde{\Delta}_{m_*}} \leq \sum^{K-1}_{i=1} 4 \cdot 2^{m_i+1} \leq \sum^{K-1}_{i=1} \frac{32}{\Delta_i}
\end{align*}

Putting all the three cases together, we get the desired regret bound.
\end{proof}

\begin{remark}
Our algorithm achieves a regret where the coefficient of the $\log(T)$ term is independent of $K$, the number of arms. This is obtained by utilizing the side observations among arms. In contrast, 
 a direct application of the UCB based policy for time-associative bandits in~\cite{continuous-time-bandit} to our problem leads to a regret of $O\left(\sum^K_{i=1}\frac{\gamma\log (T(K+1))}{\Delta^2_i}\right)$, where the $\log(T)$ term has a coefficient that is linear of $K$. 
\end{remark}


\begin{remark}
In our numerical study, we also consider a variant of Algorithm~\ref{alg:stochastic-finite} where in arm elimination phase, we delate all the periods $x_i \in X_m$ such that $\overline{l}_{i}-x_i\overline{\lambda}_{m} \geq \min_{x_j \in X_m}\overline{l}_{j}-x_j\overline{\lambda}_{m}+4c_m$. By using a smaller confidence interval, this variant eliminates suboptimal arms more aggressively than Algorithm~\ref{alg:stochastic-finite}. Although we are not able to prove a regret bound for this variant, it exhibits even better performance than Algorithm~\ref{alg:stochastic-finite} in our numerical study.
\end{remark}

\subsection{Continuous Defense Periods}
We next consider the case where the defense periods can take any real value in $X = [x_{\min}, x_{\max}]$. Note that in this case, the bound given in Theorem~\ref{thm:finite} can be very poor due to the large $K$ and small $\Delta_i$. Built upon Algorithm~\ref{alg:stochastic-finite}, we propose a new policy with a regret that is independent of $K$ and $\Delta_i$ under the following assumption. 
Let $l(x) = E_{a_1}(l(x,a_1))$ denote the expected loss when a period $x$ is played. We assume that $l(x)$ is Lipschitz continuous: there exists a constant $L \geq 0$ such that for any $x_1, x_2 \in X$, $|l(x_1)-l(x_2)| \leq L|x_1-x_2|$. For instance, when the attack time follows a uniform distribution in $[a_1,a_2]$, and $f(\cdot)$ is binary, we have $|l(x_1)-l(x_2)| \leq \frac{1}{a_2-a_1} |x_1-x_2|$, and we can take $L = \frac{1}{a_2-a_1}$. 

Our algorithm is inspired by UCB for continuous bandits (UCBC)~\cite{Auer-continuum}. We first divide $X$ into $n$ subintervals of equal length, where $n$ is a parameter to be determined (see Algorithm~\ref{alg:stochastic-infinite}). Let $x_k \triangleq x_{\min}+k\frac{x_{\max}-x_{\min}}{n}$ denote the longest period in the $k$-th interval. We then apply Algorithm~\ref{alg:stochastic-finite} to the set of arms $I \triangleq \{x_1,...,x_n\}$. 


\begin{algorithm}
\caption{Improved UCB based optimal timing with continuous periods}\label{alg:stochastic-infinite}
\begin{algorithmic}
\State {\bf Input:} A set of periods $X = [x_{\min},x_{\max}]$, the number of rounds $T$, the number of subintervals $n$.
\State {\bf Initialization:} For $k = 1, 2, ..., n$, $x_k = x_{\min}+k\frac{x_{\max}-x_{\min}}{n}$.
\State Apply Algorithm~\ref{alg:stochastic-finite} to $I = \{x_1,x_2...,x_n\}$.
\end{algorithmic}
\end{algorithm}

Define $I_1 \triangleq [x_{\min},x_1]$ and $I_k \triangleq (x_{k-1},x_k]$ for $1 < k \leq n$. For any $x \in X$, let $\Delta(x) \triangleq l(x)-x\lambda^*$ denote the relative loss of $x$. Among $x_k \in I$, assume $x_{k^*}$ has the minimum $\lambda(x_k)$. Let $\Delta'(x) \triangleq l(x)-x\lambda(x_{k^*})$ denote the relative loss of arm $x$ with respect to $x_{k^*}$. It is clear that $\Delta'(x) \leq \Delta(x)$. 
We further have the following property about $\Delta$ and $\Delta'$.
\begin{lemma}\label{lem:infinite}
$\Delta(x_{k^*}) \leq L'n^{-1}$ and $\Delta(x)-\Delta'(x) \leq L'n^{-1}$, where $L' = L\frac{x_{\max}(x_{\max}-x_{\min})}{x_{\min}}$.
\end{lemma}

From the lemma, we can establish the following performance bound for Algorithm~\ref{alg:stochastic-infinite}. 
\begin{theorem}\label{thm:infinite}
The expected regret of the variant of the improved UCB policy for continuous bandits described in Algorithm~\ref{alg:stochastic-infinite} is at most $3L'n^{-1}T + \frac{48\gamma\log(T(n+1))}{L'n^{-1}}+\frac{48 n^2}{L'}+n\Delta_{\max}$. By taking $n = T^{1/3}$, we have $\overline{R}_T \leq O(T^{2/3})$.
\end{theorem}

Proofs of Lemma~\ref{lem:infinite} and Theorem~\ref{thm:infinite} are provided in \iftp the Appendix.\else \cite{technical-report}.\fi







\begin{figure}[t]
\centering
\subfigure[binary loss]{
\label{Fig.binary}
\includegraphics[width=1.6in]{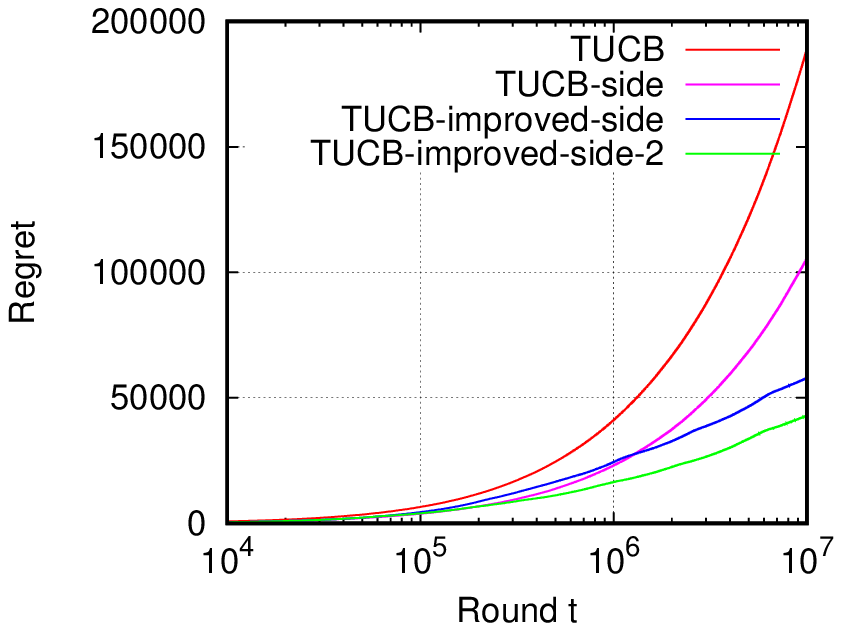}}
\subfigure[linear loss]{
\label{Fig.linear}
\includegraphics[width=1.6in]{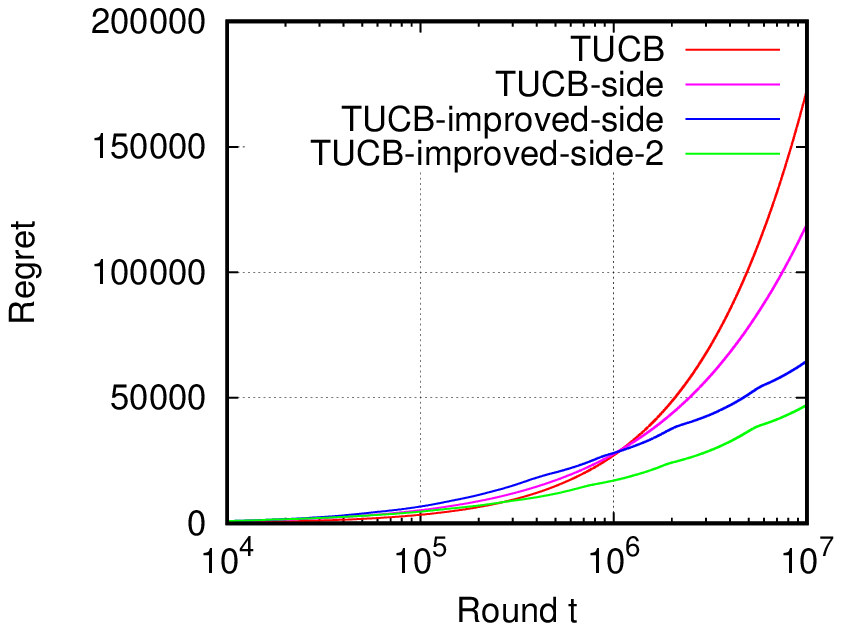}}
\caption{\small Numerical Results.}
\label{Fig.result}
\end{figure}

\section{Numerical Results}
In this section, we demonstrate the advantages of our learning algorithms through numerical study. We use the following synthetic dataset. We assume that the attack time $a_t$ follows an $i.i.d.$ Weibull Distribution with CDF $F(a) = 1-e^{-(a/\lambda)^b}$ for $a \geq 0$ and $F(a) = 0$ for $a<0$. This model has been used in reliability engineering~\cite{Bayesian-replacement} and cybersecurity~\cite{Yue-2016} to model failure times. Note that when $b = 1$, the Weibull Distribution becomes the exponential distribution. By setting $b>1$, the model indicates that the failure rate increases with time. We set $b=2$ in experiments. In each trial, $\lambda$ is chosen from the interval $[1,20]$ uniformly at random. We consider a 19 arm setting with $x_i$ evenly distributed in $[1,10]$ with a step size of 0.5. We consider both the binary loss function and the linear loss function mentioned in the model section. In both cases, we fix the defense cost to $c_d = 0.1$. With these parameter settings, we observe that the best arm varies over the feasible defense periods when we vary $\lambda$.

We focus on the case where side observations are available (without attack cost) and compare our algorithms with the UCB based time-associative bandit algorithm in~\cite{continuous-time-bandit} (TUCB) that do not consider side observations. We further consider a variant of TUCB that uses the TUCB policy to choose the arm to play in each round and obtains side-observations after each play (TUCB-side). This algorithm can be considered as the application of the UCB-N policy and the UCB-MaxN policy in~\cite{side-stochastic-12} to the time associative bandit model (UCB-N and UCB-MaxN give the same policy under the dependence structure we consider.) For our algorithms, we evaluate both Algorithm 1 (TUCB-improved-side) and its variant discussed above (TUCB-improved-side-2). The results are averaged over 100 independent trials and are given in Figure~\ref{Fig.result}. We note that the linear loss setting represents the harder case since it introduces smaller variances across arms. We observe that for both loss functions, our algorithms can significantly reduce long-term regrets compared to TUCB and TUCB-side by carefully incorporating side observations. Moreover, TUCB-improved-side-2 achieves the best performance among the four algorithms. 

\section{Acknowledgments}
This research was supported in part by a grant from the Army Research Office AROW911NF-15-1-0277, and an Army Research Office MURI W911NF-12-1-0385.

\Urlmuskip=0mu plus 1mu\relax
\bibliographystyle{aaai}
\bibliography{ref}

\iftp

\section{Appendix}
\subsection{Proof of Lemma 3}

Assume $\lambda^* = \lambda(x^*)$ for $x^* \in I_j$. By the definition of $k^*$, we have
$\lambda(x_{k^*}) = \frac{l(x_{k^*})}{x_{k^*}} \leq \frac{l(x_j)}{x_j} = \lambda(x_j)$. It follows that
\begin{align*}
\Delta(x_{k^*}) &= l(x_{k^*}) - x_{k^*}\lambda^* \\
&\leq l(x_j) \frac{x_{k^*}}{x_j} - x_{k^*} \lambda^* \\
&= \frac{x_{k^*}}{x_j}(l(x_j) - x_j \lambda^*) \\
&\leq \frac{x_{\max}}{x_{\min}}\Delta(x_j).
\end{align*}
Moreover,
\begin{align*}
\Delta(x_j) &= l(x_j) - x_j \lambda^* \\
&= l(x_j) - x_j \frac{l(x^*)}{x^*} \\
&\leq l(x_j) - l(x^*) \\
&\leq L|x_j-x^*| \\
&\leq L\frac{x_{\max}-x_{\min}}{n}
\end{align*}
Therefore, $\Delta(x_{k^*}) \leq L\frac{x_{\max}}{x_{\min}}\frac{x_{\max}-x_{\min}}{n}$. Similarly, $\Delta(x)-\Delta'(x) = x(\lambda(x_{k^*})-\lambda(x^*)) \leq x(\lambda(x_j)-\lambda(x^*)) = x(\frac{l(x_j)}{x_j}-\frac{l(x^*)}{x^*}) \leq \frac{x}{x^*}(l(x_j)-l(x^*)) \leq \frac{x_{\max}}{x_{\min}}L\frac{x_{\max}-x_{\min}}{n}$.

\subsection{Proof of Theorem 2}
We split the set of intervals into two parts. 
Define $S \triangleq \{k \in I: \Delta(x_k) > 2L'n^{-1}\}$. We have $\overline{R}_T = \sum^n_{k=1} \Delta(x_k) E(n_k(T))$ where $n_k(T)$ denotes the number of times that $x_k$ is played. Let $\overline{R}_T = \overline{R}_{T,1} + \overline{R}_{T,2}$, where $\overline{R}_{T,1} = \sum_{k \not \in S}\Delta(x_k) E(n_k(T))$ and $\overline{R}_{T,2} = \sum_{k \in S} \Delta(x_k) E(n_k(T))$. It is easy to see that $\overline{R}_1 \leq 2L' n^{-1} T$. On the other hand, we have
\begin{align*}
\overline{R}_{T,2} &= \sum_{k \in S} \Delta(x_k) E(n_k(T)) \\
&\leq \sum_{k \in S} (\Delta'(x_k)+L'n^{-1}) E(n_k(T)) \\
&\leq \sum_{k \in S}\Delta'(x_k)E(n_k(T)) + L'n^{-1}T
\end{align*}
By Theorem 1, we have $\sum_{k \in S}\Delta'(x_k)E(n_k(T)) \leq \frac{48\gamma\log(T(n+1))}{\min_{k \in S} \Delta'(x_k)}+\sum_{k \in S}\left(\Delta'(x_k)+\frac{48}{\Delta'(x_k)}\right)$. By Lemma 3, $\Delta'(x_k) \geq \Delta(x_k) - L'n^{-1} \geq L'n^{-1}$ for any $k \in S$. It follows that, $\overline{R}_T \leq 3L'n^{-1}T + \frac{48\gamma\log(T(n+1))}{L'n^{-1}}+\frac{48 n^2}{L'}+n\Delta_{\max}$. By taking $n = T^{1/3}$, we have $\overline{R}_T \leq O(T^{2/3})$.

\subsection{Costly Attacks}
Our model and solutions can be extended to a myopic attacker with a hidden attack cost $c_a$. After observing the defender's move in the beginning of round $t$, the attacker attacks immediately if $\mathbb{E}_{a_t}[(x_t-a_t)^+]> c_a$, and does not attack in that round otherwise
That is, it attacks only if its expected benefit is larger than the attack cost. Equivalently, we can assume there is a period $x_0$ such that the attacker attacks only when $x_t > x_0$. We can further distinguish the following two cases.

\begin{itemize}
\item{\bf Fixed Attack Cost:} In this case, there is a fixed $x_0$ that is unknown to the defender such that there is no attack in round $t$ if $x_t \leq x_0$ is played (thus the defender only suffers from the defense cost). The defender's loss in round $t$ is defined as:
\begin{equation}
l(x_t,a_t) = \left\{
\begin{array}{ll} \label{eq:loss2}
c_d & \text{if } x_t \leq x_0,\\
f[(x_t-a_t)^+]+c_d & \text{if } x_t > x_0.
\end{array} \right.
\end{equation}
Note that by setting $x_0 \leq x_{\min}$, this case reduces to the setting when there is no attack cost.  
\vspace{0.5ex}
\item{\bf Random Attack Cost:} In this case, $x_0$ is $i.i.d.$ sampled in each round from an unknown distribution. The defender's loss in round $t$ is same as~\eqref{eq:loss2}.
\end{itemize}

When there is a fixed attack cost, playing a longer period $x_i$ provides side-observation to a shorter arm $x_j$ only when $x_j > x_0$. Since $x_0$ is unknown, it is insufficient to only play the longest period in each round as we did before. Therefore, we modify Algorithm~\ref{alg:stochastic-finite} by maintaining a set of periods $Y_m$ that the defender knows to be longer than $x_0$ by stage $m$. We set $Y_0 = \emptyset$. Since $x_0$ is fixed, whenever an attack is observed when playing a period $x_i$, we know $x_j > x_0$ for all $x_j \geq x_i$. 
In each stage $m$, for each active period in $Y_m$, the algorithm plays the longest one only as before to exploit side-observations. However, each active period not in $Y_m$ requires further exploration. Therefore, they are also played the same number of times in each stage. But whenever an attack is observed on $x_i$, all $x_j \geq x_i$ are added to $Y_m$ and they don't need to be played separately any further. We can prove the following regret bound for the modified algorithm.

\begin{theorem}\label{thm:unknow-cost}
The expected regret of the above algorithm is at most $\sum_{i: x_i \leq x_0}\frac{B_i}{\Delta_i}+ \sum_{i: x_i > x_0}\min\left(\frac{B_i}{\Delta_i},\frac{\Delta_i}{p_i}\right)+ \frac{48\gamma\log\left(T(K+1)\frac{\Delta^2_{\max}}{4}\right)}{\Delta_{\min}}+\sum_{i: \Delta_i>0}\left(\Delta_i+\frac{48}{\Delta_i}\right)$, where $B_i = 32\gamma\log\left(T(K+1)\frac{\Delta^2_i}{4}\right)$, and $p_i \triangleq \mathbb{P}(l(x_i,a_t)>0)$. 
\end{theorem}
\begin{proof}
We adopt a similar argument as in the proof of Theorem~\ref{thm:finite}. In particular, the last two terms in~\eqref{regret-terms} can be bounded using the same argument. The only difference is in the first term. For every period $x_i < x_0$, we bound its regret up to stage $m_i$ by $\Delta_i n_{m_i} \leq \Delta_i\left\lceil \frac{2\gamma\log(T(K+1)\tilde{\Delta}^2_{m_i})}{\tilde{\Delta}^2_{m_i}} \right\rceil < \Delta_i \left(1+ \frac{32\gamma\log(T(K+1)\frac{\Delta^2_i}{4})}{\Delta^2_i}\right)$, which gives the first term (and part of the last term) in the regret. Next consider a period $x_j > x_0$. Let $p_j \triangleq \mathbb{P}(l(x_j,a_t)>0)$ 
denote the probability that an attack is observed when playing $x_j$. 
From the algorithm, the expected number of rounds until $x_j$ is added to $Y_m$ is bounded by $\min(\frac{1}{p_j},n_{m_j})$, which gives the second term in the regret (and part of the last term). After $x_j$ is added to $Y_m$, it is played only when it becomes the longest active period in $Y$, which gives the third term in the regret using the same argument for the first term of~\eqref{regret-terms} in Theorem~\ref{thm:finite}.
\end{proof}

When $x_0$ is $i.i.d.$ sampled from an unknown distribution, playing longer periods do not provide deterministic side-observations to any shorter arms. This can be addressed by playing every active arm in $X_m$ until round $n_m = \left\lceil \frac{2\gamma\log(T(K+1)\tilde{\Delta}^2_m)}{\tilde{\Delta}^2_m}\right\rceil$ in each stage $m$. The algorithm applies to any time associative stochastic bandit problem with arbitrary $l_t(x_t,a_t)$. By applying a similar argument as in the proof of Theorem~\ref{thm:finite}, the algorithm achieves a regret bound of $\sum^K_{i=1}\frac{32\gamma\log\left(T(K+1)\frac{\Delta^2_i}{4}\right)}{\Delta_i}+\sum_{i: \Delta_i>0}\left(\Delta_i+\frac{48}{\Delta_i}\right)$. 

\fi

\end{document}